\documentclass[conference]{IEEEtran}
\IEEEoverridecommandlockouts

\usepackage{bm}
\usepackage{algorithm}
\usepackage{multirow}
\usepackage{booktabs}
\usepackage[caption=false, font=footnotesize]{subfig}
\usepackage{cite}
\usepackage{amsmath,amssymb,amsfonts}
\usepackage{algorithmic}
\usepackage{graphicx}
\usepackage{textcomp}
\usepackage{xcolor}

\newtheorem{assumption}{\textbf{Assumption}}
\newtheorem{theorem}{\textbf{Theorem}}
\newtheorem{lemma}{\textbf{Lemma}}
\newtheorem{remark}{Remark}

\usepackage{graphics}

\allowdisplaybreaks

\def\BibTeX{{\rm B\kern-.05em{\sc i\kern-.025em b}\kern-.08em
    T\kern-.1667em\lower.7ex\hbox{E}\kern-.125emX}}

\begin{document}

\title{FedKRSO: Communication and Memory Efficient Federated Fine-Tuning of Large Language Models\vspace{-0.5em} \\
\thanks{G. Yang and T. Wu contributed equally to this work. Y. Sun and Y. Gong are the co-corresponding authors. The work of G. Yang, Y. Guo, and Y. Gong was supported in part by the U.S. National Science Foundation under Grants CNS-2047761, CNS-2106761, and CNS-2318683. The work of T. Wu and Y. Sun was supported in part by a Seed Grant award from the Institute for Computational and Data Sciences (ICDS) at the Pennsylvania State University (PSU). This content is solely the responsibility of the authors and does not necessarily represent the views of the ICDS.} 
}

\author{\IEEEauthorblockN{Guohao Yang\textsuperscript{$\ast$}, Tongle Wu\textsuperscript{$\star$}, Yuanxiong Guo\textsuperscript{$\ast$}, Ying Sun\textsuperscript{$\star$}, and Yanmin Gong\textsuperscript{$\ast$}}
\IEEEauthorblockA{\textsuperscript{$\ast$}University of Texas at San Antonio, San Antonio, TX 78249 \\
\textsuperscript{$\star$}Penn State University, University Park, PA 16802 \\
Emails: guohao.yang@utsa.edu, tfw5381@psu.edu, yuanxiong.guo@utsa.edu, ybs5190@psu.edu,  yanmin.gong@utsa.edu}
\vspace{-3.0em}
}

\maketitle

\begin{abstract}
Fine-tuning is essential to adapt general-purpose large language models (LLMs) to domain-specific tasks. As a privacy-preserving framework to leverage decentralized data for collaborative model training, Federated Learning (FL) is gaining popularity in LLM fine-tuning, but remains challenging due to the high cost of transmitting full model parameters and computing full gradients on resource-constrained clients. While Parameter-Efficient Fine-Tuning (PEFT) methods are widely used in FL to reduce communication and memory costs, they often sacrifice model performance compared to FFT. This paper proposes FedKRSO (Federated $K$-Seed Random Subspace Optimization), a novel method that enables communication and memory efficient FFT of LLMs in federated settings. In FedKRSO, clients update the model within a shared set of random low-dimension subspaces generated by the server to save memory usage. Furthermore, instead of transmitting full model parameters in each FL round, clients send only the model update accumulators along the subspaces to the server, enabling efficient global model aggregation and dissemination.  
%
%
%
By using these strategies, FedKRSO can substantially reduce communication and memory overhead while overcoming the performance limitations of PEFT, closely approximating the performance of federated FFT.  
The convergence properties of FedKRSO are analyzed rigorously under general FL settings. Extensive experiments on the GLUE benchmark across diverse FL scenarios demonstrate that FedKRSO achieves both superior performance and low communication and memory overhead, paving the way towards on federated LLM fine-tuning at the resource-constrained edge. 
\end{abstract}


\section{Introduction}

Large Language Models (LLMs) have demonstrated impressive capabilities in a wide range of real-world applications, including question answering, text summarization, and dialogue generation~\cite{zhao2023survey}. Despite their strong generalization capabilities, task-specific fine-tuning is often necessary to further enhance their responsiveness and alignment with downstream objectives. %
%
%
%
As task-specific data are typically privacy-sensitive and distributed across multiple clients, Federated Learning (FL)~\cite{mcmahan2017communication} that enables collaborative model training without centralizing data has received increasing attention for LLM fine-tuning. However, applying FL to LLM fine-tuning faces significant challenges. In particular, the massive number of parameters in modern LLMs makes {full-parameter fine-tuning (FFT)} memory, computation, and communication-intensive, rendering it impractical for deployment in resource-constrained edge environments \cite{zhang2022scalable, guo2022hybrid,zhang2024heterogeneity,gao2025heterogeneity}.


To address the above challenge, recent studies have Parameter-Efficient Fine-Tuning (PEFT) techniques within the FL framework \cite{infocommzhenxiao, yeh2023navigating, sun2024improving, zhang2024towards, wang2024flora, yan2024frlora, bai2024federated, babakniya2023slora}. In particular, as the most commonly used PEFT method, Low-Rank Adaption (LoRA) \cite{hu2022lora}  adapts only a small subset of model parameters. %
%
%
While effective in reducing memory footprint, LoRA often suffers from a performance drop compared to FFT, especially in the presence of heterogeneous data distributions in FL\cite{babakniya2023slora}. It has been observed that the limited capacity of low-rank updates in LoRA makes it difficult to fully capture the global model dynamics, leading to slower convergence or sub-optimal accuracy \cite{chen2022revisiting, pu2023empirical}.

To achieve strong performance under federated settings in resource-constrained edge environments,
we propose a new federated fine-tuning method that integrates subspace optimization with a finite seed-based random projection mechanism. Our approach named \textbf{FedKRSO} (\textbf{Fed}erated \bm{$K$}-Seed \textbf{R}andom \textbf{S}ubspace \textbf{O}ptimization) maintains the resource efficiency benefits of PEFT while more closely approaching the model performance of FFT. %
%
%
%
%
Specifically, FedKRSO includes two key components: First, during local training in a FL round, each client compresses the full gradients by projecting them onto a shared finite set of random subspaces and then uses the compressed gradients to optimize the model. These subspaces are generated by the server to ensure that the compression is unbiased so that the local updates take the same form as FFT in expectation. By using compressed gradients instead of full gradients, the memory usage at each client is greatly reduced. Second, after local training in a FL round, each client merges the local model updates along each subspace and uploads the total model updates to the server. The server then aggregates the model updates from all clients subspace by subspace and send them back to the clients. This strategy enables clients to reconstruct the global model used in the next round, while reducing the communication cost for both uplinks and downlinks. Overall, FedKRSO can achieve communication and memory efficiency similar to PEFT-based approaches, while maintaining model performance comparable to FFT. %
%
%
%
%
%

In summary, our key contributions are as follows:
\begin{itemize}
    \item We introduce FedKRSO, a novel method that employs random subspace optimization with a finite set of random seeds, to achieve communication and memory efficient federated fine-tuning of LLMs at the edge, while maintaining high model performance. 
    
    
    \item We perform rigorous theoretical analyses of FedKRSO and establish its convergence rates under non-convex loss and heterogeneous data distributions in federated settings.  
    
    \item We conduct extensive experiments across multiple models, datasets and data heterogeneity settings to demonstrate the advantages of FedKRSO over state-of-the-art baselines in terms of both model performance and communication and memory efficiency. 
\end{itemize}

\section{Related Work}\label{sec: related_work}

\subsubsection{Memory-Efficient Centralized Optimization Methods for LLMs}

PEFT methods are essential for adapting large pre-trained models under resource constraints. Early techniques like Adapter~\cite{houlsby2019parameter} and Prefix-Tuning~\cite{li2021prefix} enable task-specific adaptation with small parameter updates. LoRA {has emerged} as the most widely adopted method. These methods either update or add a relatively small number of tunable parameters to reduce the memory and computation cost. As a trade-off, such a restriction may prevent the algorithms from finding models with high accuracy that lie outside the searchable space. {It has been noted} in several recent works that they may not reach a performance comparable to {FFT} \cite{lialin2023scaling, chen2022revisiting, pu2023empirical}.

Zeroth-order (ZO) optimization achieves memory efficiency by estimating gradients using finite differences of function values \cite{malladi2023fine}. Various ZO algorithms {have been developed} by adapting first-order (FO) methods, including ZO-SGD \cite{liu2019signsgd} and ZO-Adam \cite{chen2019zo}. Although conceptually simple, these approaches often suffer from high variance and slow convergence, particularly due to the high dimensionality of the underlying models~\cite{zhang2024revisiting}.

Algorithms more related to our local training procedure are subspace methods, such as~\cite{hao2024flora,zhao2024galore,chen2025memory,liu2025optimization,he2024subspace}, which aim to bridge the gap between PEFT and {FFT} by approximating its updates. These methods either compress the full gradient or optimize the model restricted to randomly generated low-dimensional subspaces. Although doing so reduces the memory for storing the optimizer states, they still require computing and storing the full gradient once in a while. This will result in significant peak memory usage and limit their applicability to edge devices with memory constraints, whereas our method overcomes this limitation. We refer the readers to Remark~\ref{rmk:flora} and~\ref{rmk:RSO} for more details.

\subsubsection{Federated Fine-Tuning for LLMs}

Building on  the memory-efficient centralized  methods introduced earlier, various federated fine-tuning methods of LLMs {have been proposed}, which can be roughly divided into three categories. 

One line of work integrates LoRA into FL to reduce communication cost~\cite{infocommzhenxiao, sun2024improving, zhang2024towards, wang2024flora, yan2024frlora, bai2024federated, babakniya2023slora}, and can be categorized by whether they assume homogeneous or heterogeneous LoRA ranks among clients. In the homogeneous setting, all clients use the same LoRA rank during local model training. For example, FedIT~\cite{zhang2024towards} and its variants~\cite{yan2024frlora, babakniya2023slora} integrate LoRA-based local fine-tuning of LLMs directly into standard FL procedures. In the heterogeneous setting, methods such as FLoRA~\cite{wang2024flora} and FlexLoRA~\cite{bai2024federated} assign different LoRA ranks to clients to handle system heterogeneity. FedSA-LoRA \cite{guo2025selective} shares only one of the LoRA low-rank matrices with the server for aggregation, while keeping the other matrix local to each client to handle heterogeneity. 

Another line of work applied ZO during local model training in FL to achieve memory efficiency\cite{fang2022communication}. For example, FwdLLM~\cite{xu2023fwdllm} eliminates the need for gradient computation, enabling scalable training on mobile devices. Similarly, FedKSeed~\cite{qin2023federated} proposes a seed-reuse technique to reduce per-round communication costs. These methods share similar limitations to their centralized counterparts.

There are recent works~\cite{zhaoseparate,shu2024ferret} that adopt the idea of subspace compression. In Ferret~\cite{shu2024ferret}, clients perform {FFT} for the local training and  project the model changes for each communication round to a finite set of random directions, making it sufficient to send the projected coordinates to the server. SEPARATE~\cite{zhaoseparate} also computes the full gradients locally and applies gradient compression with error-feedback for the communication step. The variable introduced to compensate for the compression error is of the same size as the full gradient. These methods focus on improving communication efficiency, but put less emphasis on memory costs.

\section{Preliminaries}\label{sec: background}

\subsection{Federated Fine-Tuning}

Consider an FL system for fine-tuning an LLM that consists of a server and $N$ clients (e.g., edge devices). Each client $n \in [N]$ has a local private dataset $\mathcal{D}_n$. Federated fine-tuning aims at collaboratively tuning model $W$ with the pre-trained weight $W^0 \in \mathbb{R}^{d_m \times d_n}$ at initialization, which can be formulated as
\begin{equation}\label{eq: original_fl}
\min_{W}~\left\{F(W):= \frac{1}{N} \sum_{n=1}^NF_n(W) \right\},
\end{equation}
where $F(W)$ is the global objective function, and $F_n(W) := \mathbb{E}_{\xi_n \sim \mathcal{D}_n}[f_n(W;\xi_n)]$ is the local objective function of client $n$ with $\xi_n$ being a datapoint sampled from distribution $\mathcal{D}_n$ and $f_n (W;\xi_n)$ being the loss of model $W$ evaluated on $\xi_n$, respectively.

The classical FL algorithm, FedAvg \cite{mcmahan2017communication}, solves \eqref{eq: original_fl} via multiple rounds of local training followed by global model aggregation. Specifically, in each round $t$, client $n \in [N]$ performs $J$  stochastic gradient descent (SGD) steps given by
\begin{align}\label{alg:FFT}
    W^{t,j+1}_{n} = W^{t,j}_{n} - \eta \cdot \hat{\nabla} f_n(W^{t,j}_{n}), \ j = 0, \ldots, J -1,
\end{align}
where $W^{t,j}_{n}$ is the local model of client $n$ at local step $j$, and $\hat{\nabla} f_n(W^{t,j}_{n}):= \nabla f_n (W^{t,j}_{n}, \xi_n^{t,j} )$ is a mini-batch stochastic gradient of $F_n$ evaluated at $W^{t,j}_{n}$, and $\eta$ is a learning rate. After local training, the server averages all local models sent from the clients and broadcasts it, which gives the following initialization for the next round:
\begin{align}\label{eq:server-avg}
W_{n}^{t+1,0}=\frac{1}{N}\sum_{n^\prime = 1}^N W^{t,J}_{n'} \textcolor{blue}{.}  
\end{align}

However, a major challenge in applying the above federated {FFT} to LLMs lies in the extremely high communication, computation, and memory overheads due to the massive parameter size of LLMs. For example, fine-tuning a Llama 7B model using the BF16 numerical format requires at least 58 GB of memory~\cite{zhao2024galore}. 
In addition to the substantial GPU memory required to store and update the full set of model parameters, federated {FFT} also incurs significant communication overhead, as the entire model (14 GB for trainable parameters) must be transmitted between clients and the central server in each communication round. As a result, {FFT} in federated settings is both memory and communication-intensive, rendering it impractical for real-world deployment in resource-constrained edge environments. 

To address these bottlenecks, recent methods attempt to reduce the memory and communication overheads by training and transmitting only some lightweight adapters, particularly LoRA \cite{hu2022lora}. However, while LoRA-based approaches are more memory and communication-efficient, empirical results show that they typically fall short of the performance achieved by {FFT}~\cite{chen2022revisiting, pu2023empirical}. Therefore, LoRA-based federated fine-tuning methods like FedIT \cite{zhang2024towards}, though efficient, are not yet on par with federated {FFT} in terms of model performance.

\section{Design of FedKRSO}\label{sec:fl_basic}

To address the challenges of federated fine-tuning, this section proposes a novel method named FedKRSO, which is designed with two primary objectives: (1) to approach the performance of federated {FFT}; and (2) to maintain or lower the computation, communication, and memory overhead compared to federated LoRA fine-tuning. We first give an overview of the method in Sec.~\ref{sec:FedKRSO-alg} and then elaborate on the design details in Sec.~\ref{sec:subspace_training} and Sec.~\ref{sec:kseed}. The communication and memory efficiency of FedKRSO is analyzed in Sec.~\ref{sec:comm_memory}. 

\subsection{Overview of FedKRSO}\label{sec:FedKRSO-alg}
\begin{figure}[t]
    \centering
    \subfloat[The overall workflow of FedKRSO. The color of the accumulators represents different seeds, while the patterns within the shapes indicate different clients (Algorithm~\ref{alg:fedkrso}).]{
        \includegraphics[width=0.96\linewidth]{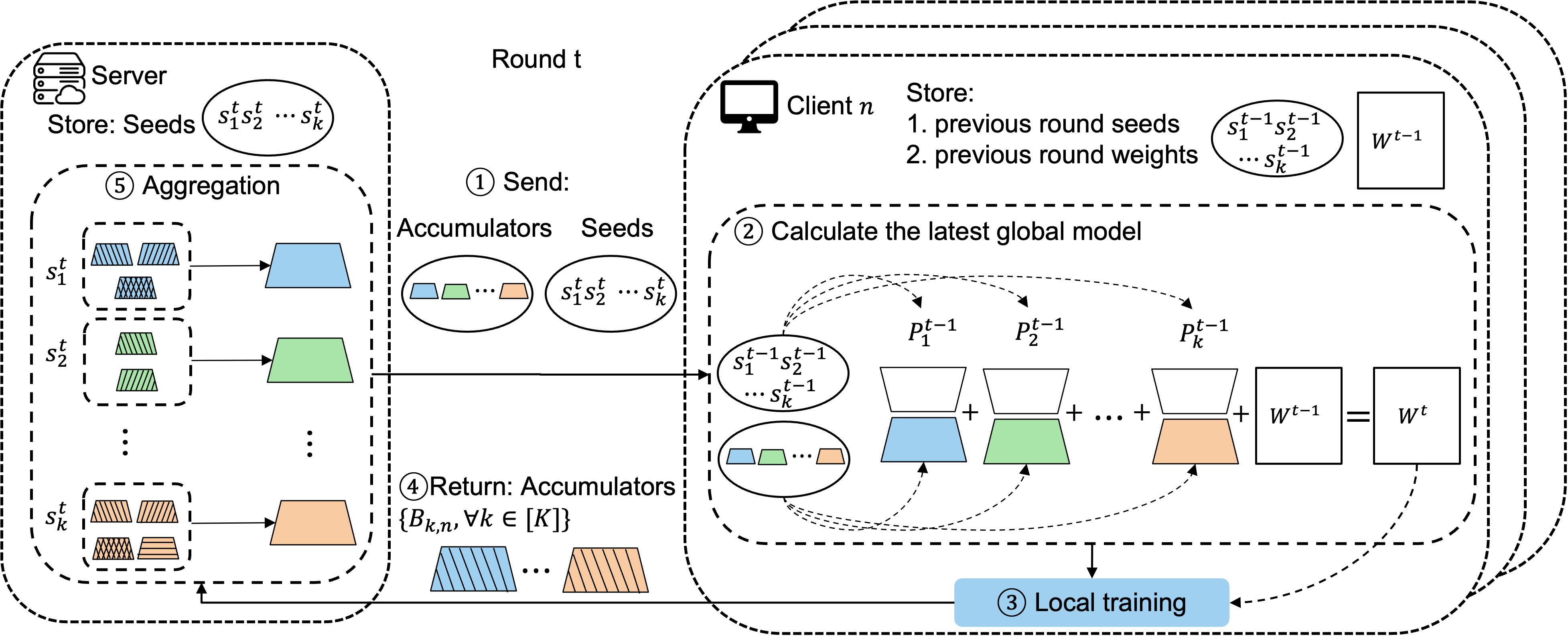}
        \label{fig:overview_a}
    }
    \hfill
    \subfloat[The memory-efficient local training process, which includes $I$ intervals, each consisting of $J$ local iterations. The reconstructed latest global model $W^t$ is used as the initialization for local training (Algorithm~\ref{alg:local training}).]{
        \includegraphics[width=0.96\linewidth]{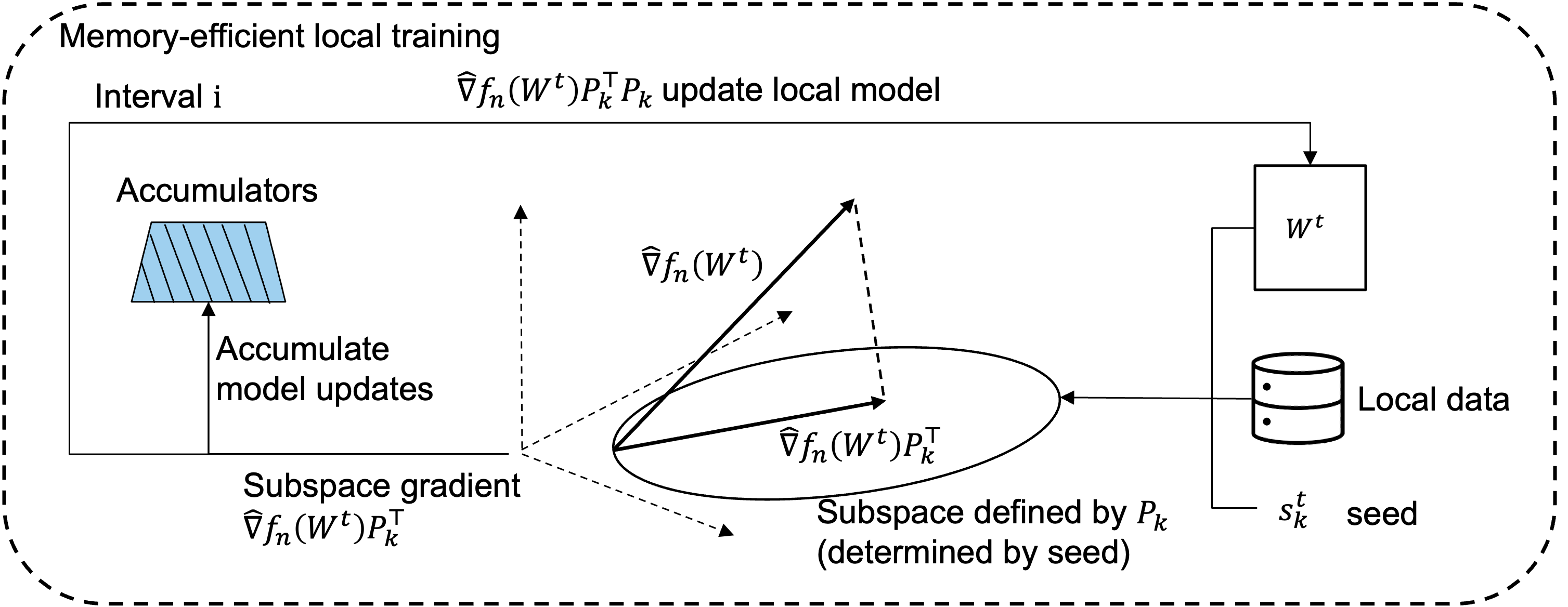}
        \label{fig:overview_b}
    }
    \caption{Overview of FedKRSO at FL round $t$.}
    \label{fig:overview}
\end{figure}

In FedKRSO, the server maintains a pool of \( K \) \emph{random seeds} for subspace generation, denoted as  $\mathcal{S} := \{s_1, s_2, \dots, s_K\}$, along with a corresponding set of $K$ model update \emph{accumulators}, each of size $d_m \times r$, denoted as $\mathcal{B} := \{{B}_1, {B}_2, \dots, {B}_K \}$. 
Both $\mathcal{S}$ and $\mathcal{B}$ are updated in each FL round. Each client, on the other hand, is initialized with a copy of the base model $W^0$. As depicted in Fig.~\ref{fig:overview}, for each FL round $t$, FedKRSO proceeds as follows:
\begin{enumerate}
\item \textbf{Seed and accumulator broadcast.} The server sends the set of random seeds \( \mathcal{S}^{t} \) and accumulators of model update \( \mathcal{B}^t \) to all clients (Line~\ref{line:broadcast} in Algorithm~\ref{alg:fedkrso}).

\item \textbf{Model reconstruction.} Each client uses \( \mathcal{B}^t \) and \( \mathcal{S}^{t-1} \) received in the \emph{previous round} to reconstruct the latest global model \( W^{t} \), which is then used to initialize its local model (Line~\ref{line:reconstruction} in Algorithm~\ref{alg:fedkrso}).

\item \textbf{Local training along random subspaces.} Local training is divided into $I$ intervals, each consisting of $J$ local iterations. %
At the beginning of each interval, client $n$ randomly samples a seed \( s_k \in \mathcal{S}^{t} \), generates the corresponding random matrix \( P_k \), and performs $J$ steps of local updates on the current model $W$ using the compressed gradient $G_B$ in the subspace defined by \( P_k \). (Lines~\ref{alg:grad-compression}-\ref{eq:alg2-W} in Algorithm~\ref{alg:local training}). Meanwhile, the compressed gradient is accumulated into the model update accumulator $B_{k, n}$ corresponding to seed \( s_k \) (Line~\ref{eq:update-grad-acc}  in Algorithm~\ref{alg:local training}). At the end of local training, we will reset the local model weights to their initial values (Line~\ref{alg:weight-reset} in Algorithm~\ref{alg:local training}) in order to reconstruct the latest global model in Step 2. 

\item \textbf{Local update uploading.} Once finishing the local training stage, each client $n$ sends at most $K$ accumulators $\mathcal{B}_n^{t+1} := \{B_{k, n}^{t+1}, \forall k \in [K]\}$ back to the server.

\item \textbf{Server-side model aggregation.}  
After receiving the updates \( \mathcal{B}^{t+1}_n \) from all clients, the server computes the global accumulator set \( \mathcal{B}^{t+1} \) by averaging the local ones \( {B}^{t+1}_{k,n} \) from all clients seed by seed (Line~\ref{line:aggregation} in Algorithm~\ref{alg:fedkrso}).
\end{enumerate}
The pseudo code for the overall procedure is shown in Algorithms~\ref{alg:fedkrso} and~\ref{alg:local training}.

\begin{algorithm}[h]
\caption{FedKRSO Algorithm}
\label{alg:fedkrso}
\begin{algorithmic}[1]
\small
\STATE \textbf{Input}: Number of clients $N$, communication rounds $T$, pretrained model weight matrix $W^0 \in \mathbb{R}^{d_m\times d_n}$, rank $r$, number of interval $I$, interval length $J$.

\STATE \textbf{Initialize}: global model update accumulators $B_k^{0} = 0$  and $P_k^{-1} = 0  \ \forall k \in [K]$.

\FOR{round $t = 0$ to $T-1$}
    \STATE Server selects a  set of $K$ random seeds $\mathcal{S}^{t}$\label{line:seed-gen}
     
    \STATE Server sends $\mathcal{B}^{t} := \{B_k^{t}\}_{k = 1}^K$ and $\mathcal{S}^{t}$ to all clients \label{line:broadcast}
    
    \FOR{client $n = 1$ to $N$ in parallel}
    
        \STATE Specify $\mathcal{P}^t := \{P_k^{t-1}\}_{k \in [K]}$ according to $\mathcal{S}^{t-1}$
        
        \STATE Reconstruct the  global model to initialize local training as 
        \[
        W^{t} = W^{t-1} + \sum_{k = 1}^K B_k^{t}  P_k^{t-1}
        \] \label{line:reconstruction}
     
        \STATE $\mathcal{B}^{t+1}_n = \textbf{LocalTraining} (W^{t}, \mathcal{S}^{t}, n)$ \label{line:local_training}

        \STATE Upload $\mathcal{B}^{t+1}_n$ to the server \label{line:upload}
    \ENDFOR\\
    
    \STATE Server aggregates:
    \[
    B^{t+1}_k = \frac{1}{N} \sum_{n=1}^{N} B^{t+1}_{k,n}, \forall k \in [K]
    \] \label{line:aggregation}
\ENDFOR

\STATE \textbf{Output}: $W^{T}$
\end{algorithmic}
\end{algorithm}

\begin{algorithm}[h]
\caption{\textbf{LocalTraining}$(W, \mathcal{S}, n)$}
\label{alg:local training}
\begin{algorithmic}[1]
\small
\STATE \textbf{Initialize:} local model update accumulators set $\mathcal{B}_n = \{B_{1, n}, B_{2, n}, \ldots, B_{K, n}\}$ with $B_{k, n} = 0, \forall k \in [K]$, number of intervals $I$, interval length $J$, learning rate $\eta$, rank $r$, decay rates $\beta_1, \beta_2$, coefficient $\epsilon$, fist-order moment $M\leftarrow 0$, second-order moment $V \leftarrow 0$.

\FOR{interval $i = 1$ to $I$}
    \STATE Uniformly sample a seed $s_k \in \mathcal{S}$ 
    \STATE Specify $P_k \sim \mathcal{N}(0, (1/r) \mathbb{I}_{r \times d_n})$ using seed $s_k$, $P_k \in \mathbb{R}^{r \times d_n}$ \label{eq:sample-seed}
    \STATE $M, V \gets 0$
    \FOR{iteration $j = 1$ to $J$}
        \STATE Sample a mini-batch from local dataset $\mathcal{D}_n$
       \STATE Compute $\hat{\nabla} f_n(W) P_k^{\top}$ according to Eq.~\eqref{eq:partial grad} and let $G_B \gets \hat{\nabla} f_n(W) P_k^{\top}$ \label{alg:grad-compression}
        \STATE $M\gets \beta_1 \cdot M+(1-\beta_1)\cdot G_B$\label{eq:update-M} 
        \STATE $V\gets \beta_2\cdot V +(1-\beta_2)\cdot G_B^2$
        \STATE $M \gets M/(1-\beta_1)$
        \STATE $V \gets V /(1-\beta_2)$
        \STATE $G_B \gets M/(\sqrt{V}+ \epsilon)$\label{eq:update-G}
        \STATE $W \gets W - \eta G_B P_k$\label{eq:alg2-W}
        \STATE $B_{k, n} \gets B_{k, n} - \eta G_{B}$\label{eq:update-grad-acc}
    \ENDFOR
\ENDFOR
\STATE Reset initial model: $W \gets W - \sum_{k = 1}^K B_{k,n}P_k$.\label{alg:weight-reset}
\STATE \textbf{Output}: $\{{B}_{k, n}, \forall k \in [K]\}$
\end{algorithmic}
\end{algorithm}


\subsection{Memory-Efficient Local Training}\label{sec:subspace_training}

In this section, we elaborate on the local training step presented in Algorithm~\ref{alg:local training}. We start with $I = 1$, which means that each client $n$ randomly selects only one seed from the pool $\mathcal{S}$. We simplify the notations here by dropping the dependency of the local variables, such as $P_k$, on the client index $n$ and round index $t$ without introducing ambiguity. Recall that in each FL round $t$, the local update of {FFT} is given by~\eqref{alg:FFT}. 

Our goal is to approximate the update~\eqref{alg:FFT} but with a memory cost comparable to LoRA, which crucially relies on computing the stochastic gradient in a memory-efficient manner. To this end, the proposed algorithm leverages the random sketching technique for gradient compression, combined with a carefully designed mechanism that enables computation of the compressed gradient \emph{without explicitly evaluating the full gradient}.\\[.2ex]

\noindent\textbf{Gradient compression.} Let $W$ be the model to be updated by client $n$ and $\hat{\nabla} f_n(W)$ be the corresponding stochastic gradient. To reduce the size of the gradient, we compress it using a linear sketch  given by
\begin{align}\label{eq:compressed_grad}
    G_B = \hat{\nabla} f_n(W) P^\top,
\end{align}
where $P \in \mathbb{R}^{r \times d_n}$ with $r \ll \min\{d_m,d_n\}$ is sampled randomly from a distribution satisfying 
$\mathbb{E}[P^\top P] = I_{d_n}$, independent from $\hat{\nabla} f_n(W)$. It is then not hard to verify that $\mathbb{E}[G_B P] = \mathbb{E}[ \hat{\nabla} f_n(W) P^\top P ] = \mathbb{E} [\hat{\nabla} f_n(W)]$ suggesting that we can reconstruct the full gradient in expectation via computing $G_B P$ and updating $W$ as
\begin{equation}\label{alg:compress-update}
W \gets W - \eta G_B P.
\end{equation}
\begin{remark}\label{rmk:local-memory-cost}
    Note that to implement~\eqref{alg:compress-update}, it suffices to store the local model $W$ and the two factors  $G_B$ and $P$, which are much smaller than the full gradient. The update of $W$ can be done in-place instead of first computing the product $G_B P$ and then storing it. Therefore, the total memory cost of this step is $d_m \times d_n + (d_m + d_n) \times r$. 
\end{remark}

\noindent\textbf{Compute the compressed gradient $G_B$.}  Notably, naively implementing the gradient compressor~\eqref{eq:compressed_grad} requires evaluating $\hat{\nabla} f_n(W)$, making the cost same as {FFT}. In what follows, we address the challenge and introduce a method for computing $G_B$ with a cost comparable to LoRA via low-rank parameterization.

We consider the loss function $F_n(W + B P)$, where \emph{constants} $W$ and $P$ are, respectively, the local model and the matrix for compression used in~\eqref{eq:compressed_grad}, and $B$ is the variable. Applying the chain rule, the gradient of $F_n$ with respect to $B$ is  given by:
\begin{align}\label{eq:partial grad}
\nabla_B F_n(W + B P) = \nabla F_n(W + B P) P^\top.
\end{align}
Comparing~\eqref{eq:partial grad} to~\eqref{eq:compressed_grad}, we can see that $G_B$ can be computed by evaluating $\hat{\nabla}_B f_n(W + B P)$ on the same mini-batch of data  as~\eqref{eq:compressed_grad}, with $W$ set to be the current model and $B = 0$.

\begin{remark}
 LoRA reparameterizes $W$ as $W = BA$ to save the memory cost. The model is updated by solving $\min_{A,B}\,F_n(W^0 + BA)$, which requires computing the partial gradients of $F_n$ with respect to $A$ and $B$ for every iteration. Comparing~\eqref{eq:partial grad} to LoRA reveals that our method reduces the gradient evaluation cost by half per iteration.
\end{remark}
Repeating the iteration~\eqref{alg:compress-update} for $J$ times gives the basic SGD version of Algorithm~\ref{alg:local training} for $I = 1$, whose convergence proof will be provided in Sec.~\ref{sec:convergence}.\\[.2ex]
\noindent\textbf{Change projection direction and momentum acceleration.} We proceed to introduce two extensions of the algorithm, which give the complete description of Algorithm~\ref{alg:local training}. 

In the basic version, we use a fixed $P$ to compress the gradient. The approach can be extended to using multiple $P$s. Specifically, for each round $t$, the server generates a set of seeds $\mathcal{S}^t$ to determine these matrices (Line~\ref{line:seed-gen}-\ref{line:broadcast} in Algorithm~\ref{alg:fedkrso}). Then in local training,
for each $i \in [I]$, each client $n$ samples a seed from $\mathcal{S}$ independently uniformly at random, indexed by $k$, to generate $P_{k}$ (Line~\ref{eq:sample-seed} in Algorithm~\ref{alg:local training}). Furthermore, to accelerate convergence, we introduce state variables $M$ and $V$ updated according to Line~\ref{eq:update-M} to~\ref{eq:update-G} in Algorithm~\ref{alg:local training}. The update mimics Adam \cite{kingma2014adam} but accumulates the first and second momentum of the compressed gradient $G_B$, making their storage cost  $d_m \times r$ rather than $d_m \times d_n$. The update of local model $W$ is then given by Line~\ref{eq:alg2-W} in Algorithm~\ref{alg:local training} using the preconditioned gradient. Once the client finishes $J$ times local updates, it resamples $P_k$ and resets $M$ and $V$ to zero.

The following two remarks compare our proposed local updates to two recent fine-tuning algorithms,  both of which apply \emph{only to the centralized setting}.

\begin{remark}[Comparison to Flora  \cite{hao2024flora}]\label{rmk:flora}
    The idea of gradient compression is also used in Flora. The key difference is that  Flora still requires computing the full gradient once in a while, whereas our method directly computes the compressed gradient via~\eqref{eq:partial grad}. The significant reduction of the peak memory cost makes our algorithm suitable for implementation on resource-constrained edge devices for federated learning.
\end{remark}

\begin{remark}[Comparison to RSO \cite{chen2025memory}]\label{rmk:RSO}
  In view of~\eqref{eq:compressed_grad} and~\eqref{eq:partial grad}, we can interpret the local update~\eqref{alg:compress-update} as minimizing $F_n$ along a subspace of weights parameterized by $W = BP$. The algorithm shares similarities with RSO but has key differences. RSO minimizes Problem~\eqref{eq: original_fl} subspace by subspace via the iterates: $B_k \approx \arg\min_{B} F_n (W_k + BP_k) + \frac{1}{2 \eta^k} \|B\|_F^2$ and $W_{k+1} = B_k P_k$. To compute $B_k$, RSO requires finding the exact minimizer up to an $\epsilon$-accuracy. This requires $\eta^k$ to be chosen small enough to make the objective convex, which introduces a hyperparameter whose value is often difficult to estimate in practice. Another difference is that RSO is a double-loop algorithm with computing $B_k$ being the inner loop. The convergence accuracy of RSO is determined by $\epsilon$, which means that achieving high accuracy requires a large number of iterations for the inner loop, and hence a high computational cost.  In contrast, ours is essentially a single-loop algorithm since $J$ can be an arbitrary fixed constant (notably, $J$ can be set to 1). 
\end{remark}

We finish this section by explaining the function of model update accumulators $B_{k,n}$ along with their update given by Line~\ref{eq:update-grad-acc}. The model update in Line~\ref{eq:alg2-W} implies that for each  $i \in [I]$, model change is equal to the accumulated gradients in the interval. Since matrix $P_k$ is fixed over the $J$ iterations,  it suffices to use $B_{k,n}$ to accumulate the compressed gradients to reproduce the model change as $B_{k,n}P_k$ in this interval. The set $\mathcal{B}_n^{t+1}$ thus fully encodes the model change of each client $n$, and is sent to the server after one round of local updates. This step will be useful for implementing the server-side update described subsequently. 

\subsection{Communication-Efficient Model Aggregation and Dissemination}\label{sec:kseed}

Section~\ref{sec:subspace_training} presented a memory-efficient approximation of the local updates in federated fine-tuning given by~\eqref{alg:FFT}. Next, we describe how to perform the model aggregation and dissemination step~\eqref{eq:server-avg} without communicating the full model weights using the $K$-seed mechanism.

Since the local models of all clients are synchronized to be $W^t$ at the beginning of round $t$, the model aggregation~\eqref{eq:server-avg} can be rewritten as\footnote{Note  that $W_n^{t,0} = W^t$ for all $n$ and $t$.} $W^{t+1} = W^{t} + (1/N)\sum_{n = 1}^N \Delta W_n^t$,
$\text{where}\ \Delta W_n^t:= W^{t,J}_{n} - W^{t}$. With $\mathcal{B}_n^{t+1}$ sent by each client $n$ and the random seeds $\mathcal{S}^t$, the server can thereby compute the model change as 
\begin{equation}
\frac{1}{N}\sum_{n = 1}^N \Delta W_n^t = \frac{1}{N}\sum_{n = 1}^N \left( \sum_{k=1}^K B_{k,n}^{t+1}P_k^t\right),
\end{equation}
where $B_{k,n}^{t+1}$ denotes the value of the accumulator $B_{k,n}$ at the end of round $t$ and $P_k^t$ is the random matrix generated by seed $s_k^t$. By exchanging the order of the summation, we find that the server can first aggregate $B_{k,n}^{t+1}$ for each seed $k \in [K]$ over the clients using a global accumulator $B_k^{t+1}$ as specified by Line~\ref{line:aggregation} in Algorithm~\ref{alg:fedkrso}. 
The global model update after round $t$ is therefore 
\begin{align}\label{eq:gloal-update}
    W^{t+1} = W^t + \sum_{k = 1}^K B_k^{t+1}P_k^t.
\end{align}
Since the clients already have the information on $\mathcal{S}^t$, it suffices to broadcast the set of global accumulators $\mathcal{B}^{t+1}: = \{B_k^{t+1}\}_{k = 1}^K$ for each client to implement~\eqref{eq:gloal-update}, as given in Line~\ref{line:broadcast}-\ref{line:reconstruction} in Algorithm~\ref{alg:fedkrso}.

Lastly, we mention that since the clients modify the model weights $W$ in-place, the initial model $W^t$ will be erased at the end of  local training. Therefore, we need to recover $W^t$ to implement~\eqref{eq:gloal-update} as specified in Line~\ref{alg:weight-reset} in Algorithm~\ref{alg:local training}. 

\subsection{Communication Overhead and Memory Footprint}\label{sec:comm_memory}
\begin{table}[thbp]
\centering
\vspace*{-10pt}
\caption{Memory and per-round communication cost comparison in terms of number of parameters. Here $P = d_m \times d_n$, $L = (d_m + d_n) \times r$, $Q = d_m \times r$.}
\label{tab:memory_communication_comparison}
\begin{tabular}{lccccc}
\toprule
& & \textbf{FedIT} & \textbf{FFA-LoRA} & \textbf{FedFFT} & \textbf{FedKRSO} \\
\midrule
\multirow{3}{*}{{Mem.}} 
& Weights & $P + L$ & $P + L$ & $P$ & $P + L$ \\
& Gradients & $L$ & $Q$ & $P$ & $Q$ \\
& Opt. states & $2L$ & $2Q$ & $2P$ & $2Q$ \\
\midrule
\multirow{2}{*}{{Comm.}} 
& Uplink & $L$ & $Q$ & $P$ & $IQ$ \\
& Downlink & $L$ & $Q$ & $P$ & $KQ + K$ \\
\bottomrule
\end{tabular}
\end{table}
The theoretical communication and memory analysis for the different algorithms was conducted on a single linear layer with weight matrix $W_l \in \mathbb{R}^{d_m \times d_n}$.
We first analyze the communication cost of FedKRSO. Since each client at most selects $I$ different $P_k$ matrices and uploads the corresponding $B_{k,n}$ in each round, the uplink communication cost for each client is bounded by $I \times d_m \times r$. On the other hand, the server broadcasts $\mathcal{B}^t$ and $\mathcal{S}^t$, which gives communication overhead of $K \times d_m \times r + K$ parameters.  

As for the memory footprint, the cost on the client side for storing the weights is $d_m \times d_n + (d_m + d_n) \times r$ as given in Remark~\ref{rmk:local-memory-cost}, and that for the compressed gradients and optimizer states is $3 \times d_m \times r$ \footnote{Here we have not accounted for the cost of activations, which depends on the model architecture and the batch size.}. A comparison with existing methods is provided in Table~\ref{tab:memory_communication_comparison}.

\section{Convergence Analysis}\label{sec:convergence}
In this section, we establish the convergence guarantee of FedKRSO in Algorithm~\ref{alg:fedkrso} under the SGD setting with  $I=1$ in Algorithm~\ref{alg:local training}.  That is, at each communication round $t$, every client $n\in[N]$ draws a single seed
$
k(n,t)\;\sim\;\mathrm{Unif}\{1,2,\dots,K\},$
and then performs $J$ steps of SGD using a random matrix denoted as $P^t_{k(n,t)}$. The generalization to cases $I > 1$ and with momentum is more involved and is left as  future work. 

To simplify notation, we omit the dependency of $k(n,t)$ on $t$, and denote it as $k(n)$ without affecting the analysis. Let $B^{t,j}_{k(n),n},\quad j=0, 1,\dots,J,\ n \in [N]$
denote the model update accumulators of client $n$ in its $j$‑th local SGD step of round $t$, using seed $k(n)$.  Then the intermediate local model is given by $W_n^{t,j}=
W^t + B^{t,j}_{k(n),n} P^t_{k(n)}$. We make the following standard assumptions on the loss functions.
\begin{assumption}[$L$-smoothness]
\label{assump:smoothness}
Each local loss function $F_n$, $n\in[N]$, is differentiable and has $L$-Lipschitz continuous gradient. That is, $\forall n \in [N], \forall $
$W_1, W_2 \in \mathbb{R}^{d_m \times d_n}$,
\begin{align} \label{eq-Lsmooth}
\bigl\|\nabla F_n(W_1) - \nabla F_n(W_2)\bigr\|_F 
\;\le\; 
L\,\bigl\|W_1 - W_2\bigr\|_F.    
\end{align}
\end{assumption}

\begin{assumption}[Bounded noise variance] 
\label{assump:stoch_grad} 
Each client $n\in[N]$ has access to unbiased stochastic gradients $\hat{\nabla} f_n(W)$ satisfying $
\mathbb{E}_{\xi_n\sim\mathcal{D}_n}\bigl[\hat \nabla f_n(W)\bigr]
= \nabla F_n(W) $ with bounded variance
\begin{align}
\mathbb{E}_{\xi_n\sim\mathcal{D}_n}\bigl\| \hat \nabla f_n(W) - \nabla F_n(W)\bigr\|_F^2
\le \sigma^2, \ \forall\,W\in\mathbb{R}^{d_m\times d_n}.
\end{align}
\end{assumption}

We consider a Non-IID setting where the data heterogeneity of the clients is characterized by the following assumption.
\begin{assumption}[Bounded gradient heterogeneity]
\label{assump:heter}
We assume that there exists constant $\varsigma$ such that $\forall \,W \in \mathbb R^{d_m \times d_n}$,
\begin{align}
 \frac 1 N \sum_{n=1}^N \left\|\nabla F_n(W)\right\|_F^2 \leq \varsigma^2 + \left\| \nabla F(W) \right\|_F^2.   
\end{align}

\end{assumption}
We further impose the following assumption on random matrices $\{P^t_k\}$, which is also used in existing works \cite{chenenhancing,chen2025memory}.  
\begin{assumption}\label{skect-assum}
The random matrices $\{P^t_k\}, \forall k \in [K], t = 1,\cdots T$, are independent and each is drawn from a distribution such that $P^t_k(P^t_k)^\top = d_n/r I_r$  and $\mathbb E [(P^t_k)^\top P^t_k] = I_{d_n}$.     
\end{assumption}
The assumption can be satisfied with high probability if the elements of $P_k^t$ are i.i.d. drawn from a centered Gaussian distribution and $\min\{d_m, d_n\} \gg r$.

With the assumptions above, we first bound the decrease of the loss function in one communication round. 
\begin{lemma}\label{lemma-1}
If Problem \eqref{eq: original_fl} satisfies Assumption~\ref{assump:smoothness} and~\ref{assump:heter}, then the iterates generated by FedKRSO under Assumption~\ref{assump:stoch_grad} and~\ref{skect-assum} with learning rate $\eta \leq \frac{r}{16Ld_nJ}$ satisfy 
\begin{align}\label{des-lemma}
\mathbb{ E} \left[ F(W^{t+1}) \right] & \leq \mathbb{ E} \left[ F(W^{t}) \right] + \frac{2Ld_nJ^2\eta^2}{r} \varsigma^2  + \frac{\eta d_n^2L^2}{r^2} Q^t \nonumber \\
& \quad - \frac{3J\eta}{8} \mathbb{E} \left[ \left\|\nabla F(W^t)\right\|_F^2\right]  + \frac{LJ\eta^2d_n^2 \sigma^2}{N r^2},
\end{align}
where $Q^t: = \frac{1}{N}\sum_{n=1}^N\sum_{j=0}^{J-1} \mathbb E \left[ \left\|   B^{t,j}_{k(n),n} P^t_{k(n)}\right\|_F^2 \right]$.
\end{lemma}
\begin{IEEEproof}
From Algorithm \ref{alg:fedkrso}, there is 
\begin{align}
 W^{t+1} = W^t - \eta \frac{\sum_{n=1}^N B^{t,J}_{k(n),n} P^{t}_{k(n)}}{N}.   
\end{align}
Then, based on the descent lemma under $L$-smoothness Assumption \eqref{assump:smoothness}, there is  
\begin{align}\label{des}
& \mathbb E \left[ F(W^{t+1})\right] \leq \mathbb E \left[ F(W^{t})  + \frac{L}{2}   \left\| \frac{\sum_{n=1}^N B^{t,J}_{k(n),n} P^{t}_{k(n)}}{N} \right\|_F^2\right] \nonumber \\
& \quad + \mathbb E \left[ \left \langle \nabla F(W^t) , \frac{\sum_{n=1}^N B^{t,J}_{k(n),n} P^t_{k(n)}}{N}\right \rangle  \right] \nonumber \\
& = \mathbb E \left[ F(W^{t}) + \frac{L\eta^2} 2  \left\| \frac{ \sum_{n,j} \nabla \hat f_n(W^{t,j}_n) (P^t_{k(n)})^\top P^t_{k(n)}} N  \right\|_F^2  \right] \nonumber\\
& \quad - \eta \mathbb E \left[ \left \langle \nabla F(W^t) ,  \frac{ \sum_{n,j} \hat \nabla f_n (W^{t,j}_n) (P^t_{k(n)})^\top P^t_{k(n)} } N \right \rangle\right] \nonumber\\
& \overset{(a)}{\leq} \mathbb E \left[ F(W^{t})\right] - \frac{J\eta} 2 \mathbb E \left[ \left\| \nabla F(W^t) \right\|_F^2\right] + \frac{LJ\eta^2d_n^2 \sigma^2}{N r^2} + \frac{\eta}{2N} \nonumber \\
& \ \cdot  \sum_{n=1}^N \sum_{j=0}^{J-1}  \mathbb E \left[  \left\|  \left( \nabla F_n ( W^{t,j}_n) \!- \!\nabla F_n(W^t)\right)(P^t_{k(n)})^\top P^t_{k(n)} \right\|_F^2 \right]  \nonumber \\
& \quad + \frac{2L\eta^2}{N^2} \left( \mathbb E \left[ \left\| \sum_{n=1}^N \sum_{j=0}^{J-1}  \nabla F_n(W^t) (P^t_{k(n)} )^\top P^t_{k(n)} \right\|_F^2 \right. \right. \nonumber \\
& \quad \left. \left. +  \left\|\sum_{n,j} \left( \nabla F_n ( W^{t,j}_n) - \nabla F_n \left( W^t \right) \right)(P^t_{k(n)} )^\top P^t_{k(n)}      \right\|_F^2 \right] \right) \nonumber \\
& \overset{(b)}{\leq} \mathbb E \left[ F(W^t)\right] - \frac{J\eta} 2 \mathbb E \left[ \left\| \nabla F(W^t) \right\|_F^2\right]  \nonumber \\
& \quad  + 2LJ^2 \eta^2 \mathbb E \left[ \left\| \sum_{n=1}^N \sum_{j=0}^{J-1}  \frac{\nabla F_n(W^t) (P^t_{k(n)})^\top P^t_{k(n)} }{NJ}\right\|_F^2 \right]  \nonumber \\
& \quad   + \left( \frac{\eta d_n^2L^2}{2r^2}  + \frac{2Jd_n^2\eta^2L^3}{r^2}\right) Q^t +  \frac{LJ\eta^2d_n^2 \sigma^2}{N r^2} \nonumber \\
&  \overset{(c)}{\leq} \mathbb E \left[ F(W^t)\right] - \frac{J\eta} 2 \mathbb E \left[ \left\| \nabla F(W^t) \right\|_F^2\right] +  \frac{LJ\eta^2d_n^2 \sigma^2}{N r^2} \nonumber \\
& \quad + \frac{2LJ\eta^2}{N} \sum_{n=1}^N \sum_{j=0}^{J-1} \mathbb E \left[ \left\| \nabla F_n(W^t) (P^t_{k(n)} )^\top P^t_{k(n)}    \right\|_F^2\right] \nonumber \\ 
& \quad + \left( \frac{\eta d_n^2L^2}{2r^2}  + \frac{2Jd_n^2\eta^2L^3}{r^2}\right) Q^t  \nonumber \\
& \overset{(d)}{=} \mathbb E \left[ F(W^t)\right]  + \left( \frac{\eta d^2L^2}{2Nr^2}  + \frac{2Jd^2\eta^2L^3}{Nr^2}\right)Q^t  + \frac{LJ\eta^2d^2 \sigma^2}{N r^2}  \nonumber \\ 
& \quad + \frac{2LJ^2\eta^2d}{Nr} \sum_{n=1}^N \mathbb E \left[ \left\| \nabla F_n(W^t)    \right\|_F^2\right] - \frac{J\eta} 2 \mathbb E \left[ \left\| \nabla F(W^t) \right\|_F^2\right]  \nonumber \\ 
& \leq \mathbb E \left[ F(W^t)\right] + \left( -\frac{J\eta} 2 + \frac{2Ld_n J^2 \eta^2 }{r}\right) \mathbb E \left[ \left\| \nabla F(W^t) \right\|_F^2\right]  \nonumber \\
& \quad +  \frac{\eta d_n^2L^2 + 4Jd_n^2\eta^2L^3}{2Nr^2} Q^t + \frac{2Ld_nJ^2\eta^2\varsigma^2}{r}  + \frac{LJ\eta^2d^2_n \sigma^2}{N r^2},
\end{align}
where (a) is due to Assumption \ref{assump:stoch_grad}, $\|a+b\|^2\leq 2(\|a\|^2 + \|b\|^2)$ and Assumption \ref{skect-assum}, which has  $\big\|P^t_{k(n)}\big\|_2 \leq \sqrt{\frac{d_n}{r}}$; (b) and (c) are based on \eqref{eq-Lsmooth} and Jensen’s inequality, respectively; and (d) is because $P^t_{k(n)}$ is independent of $W^t$ and thus 
$
\mathbb E \left[\| \nabla F_n(W^t) (P^t_{k(n)})^\top P^t_{k(n)}\|_F^2\right] = \mathbb E \left[  \textrm{Tr} \left( (P^t_{k(n)})^\top P^t_{k(n)} (\nabla F_n(W^t))^\top \nabla F_n(W^t) (P^t_{k(n)})^\top P^t_{k(n)} \right)\right]    
$ = $\frac{d_n}{r}\left\| \nabla F_n(W^t) \right\|_F^2$ by Assumption \ref{skect-assum}. 
The last inequality and \eqref{des-lemma} is due to the learning rate condition $\eta \leq \frac{r}{16Ld_nJ}$. 
\end{IEEEproof}
Our next lemma bounds  $Q^t$ in \eqref{des-lemma}.
\begin{lemma}
In the same setting of Lemma \ref{lemma-1}, FedKRSO has
\begin{align}\label{drift}
& Q^t \leq   \frac{8J^3\eta^2d_n}{r} \varsigma^2  + \frac{8J^3\eta^2d_n}{r} \mathbb E \left[ \left\|\nabla F(W^t)\right\|_F^2\right]  \nonumber \\
& \qquad + \frac{4J^2d^2_n \eta^2 \sigma^2}{r^2} 
\end{align}
\end{lemma}
\begin{IEEEproof}
Based on Line~\ref{eq:update-grad-acc} in Algorithm \ref{alg:local training}:
\begin{align}
B^{t,j+1}_{k(n),n}  =  B^{t,j}_{k(n),n} -\eta \hat \nabla f_n\left( W^{t,j}_n\right)( P^{t,k(n)})^\top.    
\end{align}
By Assumption \ref{assump:stoch_grad}, there is
\begin{align}
& \mathbb E \left[   \left\| B^{t,j+1 }_{k(n),n} P^t_{k(n)}\right\|_F^2  \right ]  =  \frac{d^2 \eta^2 \sigma^2}{r^2}  \nonumber \\
& \quad + \mathbb E \left[   \left\| \left( B^{t,j}_{k(n),n} -\eta \nabla F_n( W^{t,j}_n)( P^t_{k(n)} )^\top \right) P^t_{k(n)}\right\|_F^2  \right ]  \nonumber \\
& \overset{(e)}{\leq} \left( 1 + \frac{1}J \right) E \left[   \left\| B^{t,j}_{k(n),n} P^t_{k(n)}\right\|_F^2  \right ]  +  \frac{d^2_n \eta^2 \sigma^2}{r^2} + 4J \eta^2 \nonumber \\ 
& \quad \cdot \mathbb E \left[  \left\|  \left( \nabla F_n( W^{t,j}_n)- \nabla F_n(W^t) \right) ( P^t_{k(n)})^\top P^t_{k(n)}\right\|_F^2 \right] \nonumber \\
& \quad + 4J\eta^2 \mathbb E \left[ \left\| \nabla F_n(W^t)  ( P^t_{k(n)})^\top P^t_{k(n)} \right\|_F^2 \right] \nonumber \\
& \leq \left( 1+ \frac{1} J  + \frac{4J\eta^2 d_n^2L^2}{r^2}\right) \mathbb E \left[   \left\| B^{t,j }_{k(n),n} P^t_{k(n)}\right\|_F^2  \right ] \nonumber \\
& \quad + \frac{4J\eta^2 d_n}{r} \mathbb E \left[  \|  \nabla F_n(W^t)\|_F^2\right] + \frac{d^2_n \eta^2 \sigma^2}{r^2} \nonumber \\
& \leq \left( 1 + \frac{2} J \right)  \mathbb E \left[   \left\| B^{t,j }_{k(n),n} P^t_{k(n)}\right\|_F^2  \right ] + \frac{d^2_n \eta^2 \sigma^2}{r^2} \nonumber \\
& \quad  + \frac{4J\eta^2 d_n}{r} \mathbb E \left[  \|  \nabla F_n(W^t)\|_F^2\right],
\end{align}
where (e) uses the Young's inequality and the last inequality is due to $\eta \leq \frac{r}{16Ld_nJ}$. 

Unrolling above inequality for $j =0, \cdots, J-1$ yields
\begin{align}
\text{\small$\mathbb E \left[   \left\| B^{t,j}_{k(n),n} P^t_{k(n)}\right\|_F^2  \right ] \leq \frac{8J^2\eta^2 d_n}{r}  \mathbb E \left[  \|  \nabla F_n(W^t)\|_F^2\right] + \frac{4Jd^2_n \eta^2 \sigma^2}{r^2} $},
\end{align}
where we use fact that $\sum_{p=0}^{J-1}(1+\frac{2}{J})^p \leq 4J$. Summing up over all clients and local updates gives
\begin{align}\label{drift}
Q^t &\leq \frac{8J^3\eta^2d_n}{r} \mathbb E \left[  \frac{1}{N} \sum_{n=1}^N \left\| \nabla F_n(W^t) \right\|_F^2\right]  + \frac{4J^2d^2_n \eta^2 \sigma^2}{r^2}. 
\end{align}
Applying Assumption \ref{assump:heter} to the above inequality concludes the proof. 
\end{IEEEproof}
Combining the above two lemmas, we have the following convergence result. Denote the minimum of the objective in Problem \eqref{eq: original_fl} as $F^\star$ and the initial optimality gap as $\Delta = F(W^0) - F^\star$.
\begin{theorem}\label{thm-1}
Let Assumption \ref{assump:smoothness}-\ref{skect-assum} hold. Consider the FedKRSO given by Algorithm \ref{alg:fedkrso} under SGD setting without momentum acceleration, and with  interval number $I = 1$. Then, with the learning rate chosen such that $\eta \leq \mathcal O (\frac{r^2}{JLd^2_n})$,  the iterates generated by the algorithm satisfy
\begin{align}\label{final-conver}
& \frac{1}{T+1} \sum_{t=0}^T \mathbb E \left[ \left\| \nabla F(W^t)\right\|_F^2\right] \leq \frac{64d^2_nL\Delta}{r^2 T}   + 14 \sqrt{\frac{Ld_n\Delta}{rT}}  \varsigma \nonumber \\
& \quad \quad + 24 \left( \frac{d^2_nL\Delta\sigma }{r^2 T J^{\frac 1 2}}\right)^{\frac 2 3}  + 8 \sqrt{\frac{d^2_nL \Delta \sigma^2}{r^2NJ T}}.
\end{align}
\end{theorem}
\begin{IEEEproof}
Substituting \eqref{drift} into \eqref{des-lemma} gives
\begin{align}
& \mathbb E \left[ F(W^{t+1})\right] \leq \mathbb E \left[ F(W^t)\right]  + \frac{2Ld_nJ^2\eta^2}{r} \varsigma^2   + \frac{LJ\eta^2d^2_n \sigma^2}{N r^2}\nonumber \\
& \quad + \left( -\frac{J\eta} 2 + \frac{2Ld_n J^2 \eta^2 }{r}\right) \mathbb E \left[ \left\| \nabla F(W^t) \right\|_F^2\right] \nonumber \\
& \quad + \left( \frac{\eta d^2_nL^2}{2r^2}  + \frac{2Jd^2_n\eta^2L^3}{r^2}\right) \left(  \frac{8J^3\eta^2d_n}{r} \mathbb E \left[ \left\|\nabla F(W^t)\right\|_F^2\right] \right. \nonumber \\
& \quad \quad \left.+ \frac{8J^3\eta^2d_n}{r} \varsigma^2 + \frac{4J^2d^2_n \eta^2 \sigma^2}{r^2} \right) \nonumber \\
& \leq \mathbb E \left[ F(W^t)\right] - \frac{J\eta} 4 \mathbb E \left[ \left\| \nabla F(W^t) \right\|_F^2\right] \nonumber \\
& \quad + \left(\frac{3Ld_nJ^2 \varsigma^2}{r} + \frac{LJd^2_n\sigma^2}{Nr^2}\right) \eta^2 + \frac{4J^2L^2d^4_n\sigma^2}{r^4} \eta^3,
\end{align}
where the last inequality is due to the condition on the learning rate. Taking telescoping cancellation over $t$ gives 
\begin{align}
& \frac{1}{T+1} \sum_{t=0}^T \mathbb E \left[ \left\| \nabla F(W^t)\right\|_F^2\right] \leq \frac{4(F(W^0) - F^\star)}{J\eta T}  \nonumber \\
& \quad \quad  + \left(\frac{12 Ld_nJ \varsigma^2}{r} + \frac{4Ld^2_n\sigma^2}{Nr^2}\right) \eta + \frac{16 JL^2d^4_n\sigma^2}{r^4} \eta^2. 
\end{align}
The theorem follows by applying~\cite[Lemma 15]{koloskova2020unified}.
\end{IEEEproof}

\section{Evaluation}

In this section, we evaluate FedKRSO  comprehensively from multiple perspectives, including model performance, memory usage, and communication cost. Specifically, we first compare the model performance of FedKRSO with several baselines under different levels of data heterogeneity. Next, we then analyze its efficiency in terms of GPU memory footprint and communication overhead. Finally, we conduct an ablation study to examine the impacts of interval length $J$ and total seed number $K$ on the model performance in FedKRSO.  

\subsection{Experimental Setup}
\noindent\textbf{Datasets and models.}  We evaluate the performance of FedKRSO on GLUE~\cite{wang2018glue}: a benchmark for evaluating the performance of NLP models on eight tasks, namely,  CoLA (grammatical acceptability), SST-2 (sentiment polarity), MRPC (paraphrase detection), QQP (duplicate-question detection), STS-B (semantic-similarity regression), MNLI (three-way natural-language inference), QNLI (question-answer entailment), and RTE (binary natural-language inference). We conducted the FL experiments with $N = 10$ clients. For each task, we consider both IID and non-IID settings, where the non-IID data distributions across clients are generated using a Dirichlet distribution following \cite{hsu2019measuring}.
%
%
%
%
%
%
%
We use two pre-trained language models as base models: RoBERTa-base (125M) and RoBERTa-large  (350M)~\cite{liu2019roberta}.


\noindent\textbf{Baselines.} We consider the following three baselines for performance comparison: 
\begin{itemize}
\item \textbf{FedFFT} \cite{mcmahan2017communication}: FedAvg with FFT for local training, which serves as the upper bound for model performance but is not feasible in resource-constrained edge environments.    

\item \textbf{FedIT} \cite{zhang2024towards}: FedAvg with LoRA for local training, which is a common baseline in federated fine-tuning. 

\item \textbf{FFA-LoRA} \cite{sun2024improving}: FedIT with partially frozen LoRA modules, which serves as the state-of-the-art baseline in federated LoRA fine-tuning and offers improved memory and communication efficiency.

\end{itemize}

\noindent\textbf{System settings.} In all methods, clients perform a total of 100 local iterations in each FL round using the AdamW optimizer with a cosine decay scheduler. For all LoRA-based methods, we set the rank to 4. For our method, we set $K = 10$ and search for the optimal interval length $J$ within the set of $\{10, 20, 50, 100\}$. %

The default value for batch size in all experiments is 16, except for CoLA, which uses 32. The total number of FL rounds was 30. Each experiment is run with three random seeds, and we report the average results. All experiments are conducted on a Linux server equipped with four NVIDIA Quadro RTX 8000 GPUs using PyTorch.  

\subsection{Experimental Results}

\begin{table*}[htbp]
\centering
\caption{Performance comparison on GLUE under IID and non-IID (Dirichlet distribution with concentration parameter $\alpha$) settings using pre-trained \textbf{(a)} RoBERTa-base and \textbf{(b)} RoBERTa-large. Best and second-best results are in \textbf{bold} and \underline{underline}, respectively.}
\label{tab:glue_combined}
\scalebox{0.93}{
\subfloat[RoBERTa-base]{
\label{tab:glue_base}
\centering
\begin{tabular}{llcccc}
\toprule
\textbf{Setting} & \textbf{Task} & \textbf{FedIT} & \textbf{FFA-LoRA} & \textbf{FedFFT} & \textbf{FedKRSO} \\
\midrule
\multirow{8}{*}{IID} 
& CoLA    & 59.98 & 59.26 & \underline{60.08} & \textbf{60.28} \\
& MRPC    & 88.15 & 86.36 & \textbf{89.95}    & \underline{88.97} \\
& SST-2   & \underline{94.30} & 93.04 & \textbf{94.53} & 94.04 \\
& STS-B   & 90.66 & 90.01 & \textbf{91.11}    & \underline{90.89} \\
& QNLI    & \underline{90.05} & 86.97 & \textbf{91.45} & 89.55 \\
& QQP     & 85.16 & 82.96 & \textbf{87.22}    & \underline{85.62} \\
& MNLI-mm & 83.43 & 79.79 & \textbf{84.88}    & \underline{83.86} \\
& Average & 84.53 & 82.63 & \textbf{85.60}    & \underline{84.74} \\
\midrule
\multirow{8}{*}{$\alpha=0.5$} 
& CoLA    & 56.05 & 55.99 & \textbf{59.76}    & \underline{59.32} \\
& MRPC    & 84.64 & 80.23 & \textbf{87.91}    & \underline{86.44} \\
& SST-2   & \underline{93.12} & 92.24 & \textbf{93.27} & 92.58 \\
& STS-B   & \textbf{88.94} & 88.44 & \underline{88.83} & \underline{88.83} \\
& QNLI    & \underline{89.12} & 86.11 & \textbf{90.38} & 87.80 \\
& QQP     & 84.47 & 82.30 & \textbf{86.43}    & \underline{85.06} \\
& MNLI-mm & 81.03 & 76.81 & \textbf{83.24}    & \underline{81.87} \\
& Average & 82.48 & 80.30 & \textbf{84.26}    & \underline{83.13} \\
\midrule
\multirow{8}{*}{$\alpha=0.25$} 
& CoLA    & 56.83 & 56.34 & \textbf{59.86}    & \underline{59.45} \\
& MRPC    & 77.45 & 72.63 & \textbf{85.54}    & \underline{82.68} \\
& SST-2   & \underline{91.90} & 89.68 & \textbf{92.24} & 91.40 \\
& STS-B   & 88.70 & 87.59 & \textbf{89.10}    & \underline{89.02} \\
& QNLI    & \underline{87.03} & 82.71 & \textbf{88.36} & 85.75 \\
& QQP     & 83.09 & 80.79 & \textbf{84.89}    & \underline{83.51} \\
& MNLI-mm & 77.67 & 70.60 & \textbf{81.21}    & \underline{79.03} \\
& Average & 80.38 & 77.19 & \textbf{83.03}    & \underline{81.55} \\
\bottomrule
\end{tabular}
}
}
\scalebox{0.93}{
\subfloat[RoBERTa-large]{
\label{tab:glue_large}
\centering
\begin{tabular}{llcccc}
\toprule
\textbf{Setting} & \textbf{Task} & \textbf{FedIT} & \textbf{FFA-LoRA} & \textbf{FedFFT} & \textbf{FedKRSO} \\
\midrule
\multirow{8}{*}{IID} 
& CoLA    & 66.44 & 62.90 & \textbf{66.93}    & \underline{66.60} \\
& MRPC    & \underline{90.93} & 88.24 & \textbf{91.18} & 90.69 \\
& SST-2   & 95.64 & 95.15 & \textbf{96.41}    & \underline{95.91} \\
& STS-B   & \textbf{92.32} & 91.38 & 92.26    & \underline{92.27} \\
& QNLI    & 93.26 & 91.36 & \textbf{94.28}    & \underline{93.96} \\
& QQP     & 87.08 & 85.37 & \textbf{88.28}    & \underline{87.33} \\
& MNLI-mm & 88.51 & 86.73 & \textbf{89.25}    & \underline{88.71} \\
& Average & 87.74 & 85.87 & \textbf{88.37}    & \underline{87.92} \\
\midrule
\multirow{8}{*}{$\alpha=0.5$} 
& CoLA    & \underline{66.77} & 59.23 & 66.48         & \textbf{66.90} \\
& MRPC    & 85.78 & 79.49 & \textbf{89.05}    & \underline{88.48} \\
& SST-2   & 95.07 & 94.57 & \textbf{95.72}    & \underline{95.15} \\
& STS-B   & \textbf{91.27} & 89.93 & 90.95         & \underline{91.02} \\
& QNLI    & \underline{93.28} & 91.20 & \textbf{93.86} & 92.59 \\
& QQP     & 86.34 & 84.25 & \textbf{87.65}    & \underline{86.84} \\
& MNLI-mm & \underline{87.14} & 84.84 & \textbf{87.99} & 86.63 \\
& Average & 86.52 & 83.36 & \textbf{87.39}    & \underline{86.80} \\
\midrule
\multirow{8}{*}{$\alpha=0.25$} 
& CoLA    & 64.94 & 56.60 & \textbf{68.03}    & \underline{67.10} \\
& MRPC    & 77.29 & 73.69 & \textbf{86.60}    & \underline{85.62} \\
& SST-2   & \underline{94.95} & 94.15 & \textbf{95.03} & \textbf{95.03} \\
& STS-B   & \underline{90.80} & 88.38 & 90.27         & \textbf{90.88} \\
& QNLI    & 90.61 & 86.29 & \textbf{92.52}    & \underline{90.80} \\
& QQP     & 84.93 & 81.57 & \textbf{86.67}    & \underline{85.61} \\
& MNLI-mm & 85.03 & 81.20 & \textbf{86.58}    & \underline{85.35} \\
& Average & 84.08 & 80.27 & \textbf{86.53}    & \underline{85.77} \\
\bottomrule
\end{tabular}
}
}
\end{table*}

\subsubsection{Fine-Tuning Model Performance}

GLUE uses task-specific evaluation metrics: CoLA is evaluated with the Matthews correlation coefficient (MCC), STS-B with Pearson correlation (Corr), and all other tasks with accuracy (Acc). Table~\ref{tab:glue_combined} reports the performance of all methods under these metrics on both IID and non-IID (Dirichlet distribution with $\alpha = 0.5$ and $\alpha = 0.25$) settings, using RoBERTa-base and RoBERTa-large, respectively. %

From the results, we have the following observations. 
FedFFT consistently achieves the highest overall performance, with a noticeable margin over LoRA-based baselines. This gap further increases under non-IID settings, increasing from 1.1 (IID) to 2.65 ($\alpha = 0.5$) on RoBERTa-base, highlighting the needs for FFT, consistent with observations in \cite{babakniya2023slora}.

Second, our proposed method, {FedKRSO}, outperforms both {FedIT} and {FFA-LoRA} across all setting and better approaches FedFFT. Also, under non-IID scenarios, all methods exhibit performance degradation, which is consistent with prior findings~\cite{babakniya2023slora, guo2025selective, sun2024improving}. {However, compared to FedIT, FedKRSO experiences a smaller performance drop, achieving better performance over FedIT by 0.65 and 1.17 on RoBERTa-base under $\alpha=0.5$ and $\alpha=0.25$, respectively.} As data heterogeneity increases, the larger performance gap between FedKRSO and FedIT underscores the effectiveness of FedKRSO in highly heterogeneous data settings.


Third, RoBERTa-large experiences less performance degradation than RoBERTa-base when fine-tuned with LoRA under data heterogeneity. For instance, under IID settings, {the performance gap between FedIT and FedFFT is only 0.63 on RoBERTa-large, compared to 1.07 on RoBERTa-base.} This trend aligns with the findings in~\cite{lialin2023scaling}, suggesting that larger models are more robust to PEFT methods due to greater parameter redundancy.

\subsubsection{Memory and Communication Efficiency}

\begin{table}[htbp]
\centering
\caption{Per-round communication overhead and GPU memory usage of different approaches, measured in GB.}
\vspace*{-5pt}
\begin{tabular}{lcccc}
\toprule
\multirow{2}{*}{Approach} & \multicolumn{2}{c}{RoBERTa-base} & \multicolumn{2}{c}{RoBERTa-large} \\
\cmidrule(lr){2-3} \cmidrule(lr){4-5}
 & Comm. & Mem. & Comm. & Mem. \\
\midrule
FedIT   & 0.00958 & 2.87 & 0.0215 & 5.62 \\
FFA-LoRA     & 0.00688 & 2.70 & 0.0144 & 5.05 \\
FedKRSO & 0.02913 & 2.65 & 0.0738 & 5.00 \\
FedFFT   & 0.938   & 4.31 & 2.67   & 9.43 \\
\bottomrule
\end{tabular}
\label{tab:comparison_comm_mem}
\vspace*{-5pt}
\end{table}

Table~\ref{tab:comparison_comm_mem} summarizes the GPU memory usage and communication overhead of different approaches. The reported communication overhead includes both downlink and uplink costs, while GPU memory usage refers to the peak GPU memory footprint at each client during training, as measured by \texttt{nvidia-smi}.

From the table, we can observe that FedKRSO maintains the least peak GPU memory footprint, requiring 2.65 GB and 5.00 GB for RoBERTa-base and RoBERTa-large models, respectively. This memory efficiency primarily stemmed from FedKRSO's subspace optimization strategy as analyzed in Section~\ref{sec:subspace_training}. Also, FedKRSO demonstrated significantly lower communication costs compared to FedFFT for RoBERTa-base and RoBERTa-large. Notably, while FedKRSO incurred slightly higher communication overhead than FedIT, the gap between them was relatively small. In contrast, FedFFT required over $300\times$ higher communication cost than FedKRSO for RoBERTa-base. A similar trend holds for RoBERTa-large. These results suggest that FedKRSO achieves a strong balance between performance and communication efficiency, offering a substantial improvement over FedFFT while maintaining communication costs in the same order of magnitude as LoRA-based baselines. 

\subsubsection{Ablation Study}
In this part, we conduct an ablation study to investigate the impacts of two key hyperparameters on the model performance in FedKRSO: {interval length $J$} and {seed number $K$}. These studies aim to provide a deeper understanding of how each hyperparameter influences the effectiveness and robustness of FedKRSO. 

\begin{figure}[htbp]
    \centering
    \subfloat[Average scores across GLUE tasks with different interval lengths.]{
        \includegraphics[width=0.45\linewidth]{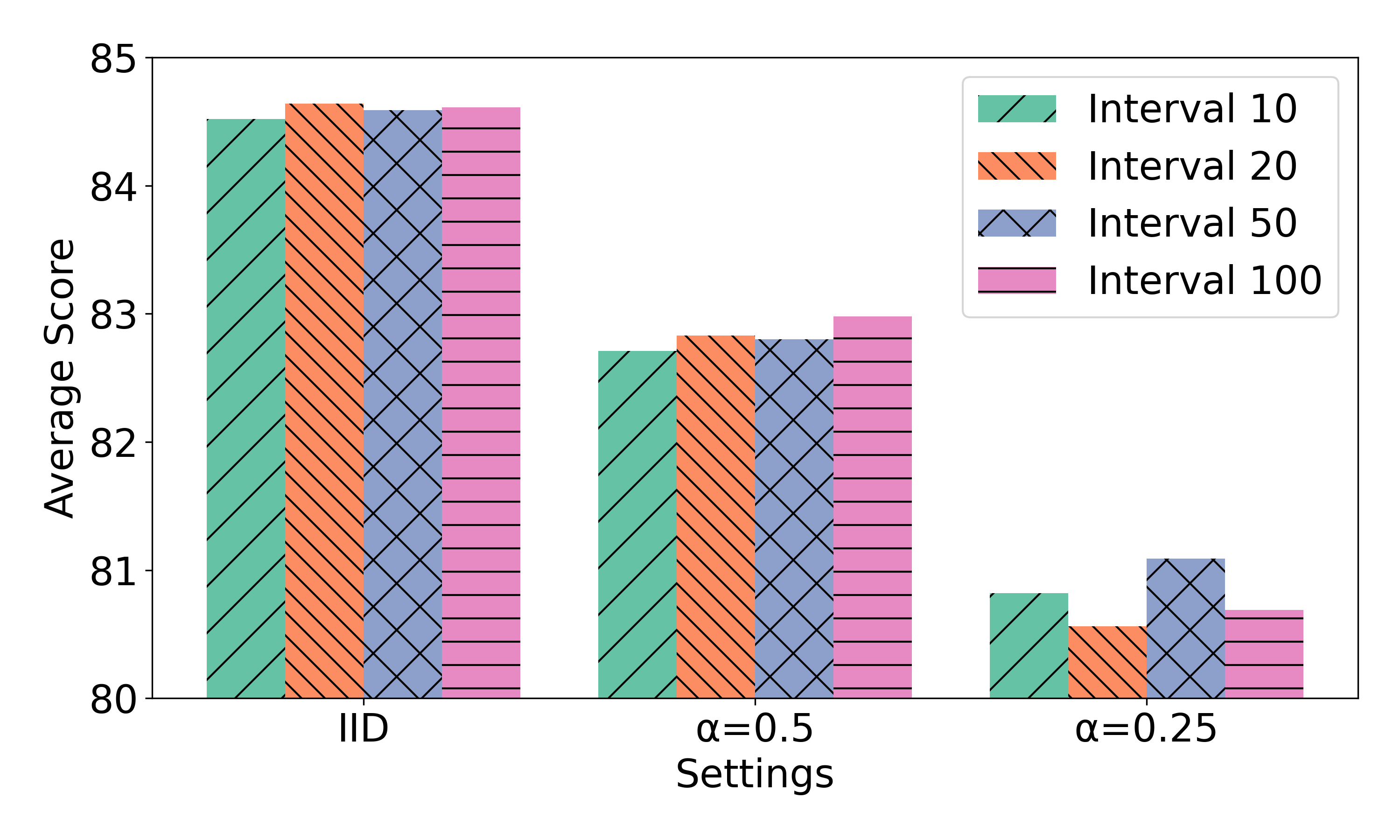}
        \label{fig:interval_ablation}
    }
    \hfill
    \subfloat[Accuracy on the MRPC task with different $K$.]{
        \includegraphics[width=0.45\linewidth]{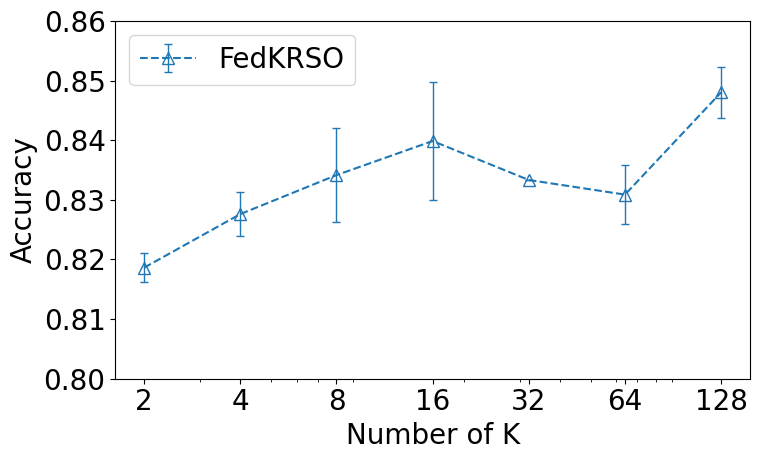}
        \label{fig:kseed_ablation}
    }

    \caption{Ablation study of FedKRSO on RoBERTa-base.}
    \label{fig:ablation}
\end{figure}

\noindent\textbf{Effect of interval length.} To investigate the effect of the local training interval on the performance of FedKRSO, we conducted an ablation study by varying the {interval length $J \in \{10, 20, 50, 100\}$ } under different data heterogeneity settings. The results are shown in Fig~\ref{fig:interval_ablation}. %

We observe that FedKRSO is robust to the choice of the {interval length} across all settings, with only minor variations in the average model performance. The interval length introduces a trade-off between subspace exploration and momentum preservation. A smaller interval enabled more frequent subspace changes, potentially enhancing parameter space exploration and improving performance in certain tasks, particularly under non-IID settings. For example, when $\alpha=0.25$, an interval of 10 achieved second best performance. Conversely, larger {interval length} (e.g., 50 or 100) reduced the frequency of subspace switching, which helps retain accumulated momentum and leads to more stable optimization dynamics. For example, when $\alpha=0.5$, an interval of 100 achieved best performance.

\noindent\textbf{Effect of seed number.} To understand the relationship between the number of random projection matrices (i.e., seeds $K$) and the {accuracy} of FedKRSO, we examine its performance under varying $K$, as shown in Figure~\ref{fig:kseed_ablation}. We observe that the {accuracy} improves as $K$ increases, suggesting that a larger number of subspaces allows the model to explore more diverse directions. When $K$ is too small, the performance of FedKRSO degrades due to limited expressiveness stemming from an insufficient number of subspaces. It is important to note that although settings such as $K=128$ yield better accuracy, they also incur higher communication costs, as more model updates need to be accumulated. Therefore, the choice of $K$ should strike a balance: it should be large enough to ensure model expressiveness, but not so large as to impose excessive communication overhead.

\section{Conclusion}
This paper proposes FedKRSO, a novel framework for federated fine-tuning of LLMs under resource-constrained edge environments. By using random subspace optimization with a finite set of random seeds, FedKRSO can enhance communication and memory efficiency while maintaining high model performance. Convergence properties of FedKRSO are provided, and experimental results demonstrate the advantages of FedKRSO compared to existing methods. In the future, we will extend FedKRSO to multimodal LLMs and evaluate its effectiveness on other benchmarks.  



\newpage
\bibliographystyle{IEEEtran}

\bibliography{cited.bib}

@article{zhao2023survey,
  title={A survey of large language models},
  author={Zhao, Wayne Xin and Zhou, Kun and Li, Junyi and Tang, Tianyi and Wang, Xiaolei and Hou, Yupeng and Min, Yingqian and Zhang, Beichen and Zhang, Junjie and Dong, Zican and others},
  journal={arXiv preprint arXiv:2303.18223},
  volume={1},
  number={2},
  year={2023}
}

@inproceedings{mcmahan2017communication,
  title={Communication-Efficient Learning of Deep Networks from Decentralized Data},
  author={McMahan, Brendan and Moore, Eider and Ramage, Daniel and Hampson, Seth and y Arcas, Blaise Aguera},
  booktitle={Artificial Intelligence and Statistics},
  vol={54},
  pages={1273--1282},
  year={2017}
}

@article{zhang2022scalable,
  title={Scalable and low-latency federated learning with cooperative mobile edge networking},
  author={Zhang, Zhenxiao and Gao, Zhidong and Guo, Yuanxiong and Gong, Yanmin},
  journal={IEEE Transactions on Mobile Computing},
  volume={23},
  number={1},
  pages={812--822},
  year={2022},
  publisher={IEEE}
}

@inproceedings{sun2024improving,
  title     = {Improving {L}o{RA} in Privacy-preserving Federated Learning},
  author    = {Youbang Sun and Zitao Li and Yaliang Li and Bolin Ding},
  booktitle = {International Conference on Learning Representations},
  year      = {2024},
}

@inproceedings{hu2022lora,
title={Lo{RA}: Low-Rank Adaptation of Large Language Models},
author={Edward J Hu and yelong shen and Phillip Wallis and Zeyuan Allen-Zhu and Yuanzhi Li and Shean Wang and Lu Wang and Weizhu Chen},
booktitle={International Conference on Learning Representations},
year={2022},
}

@inproceedings{
babakniya2023slora,
title={{SL}o{RA}: Federated Parameter Efficient Fine-Tuning of Language Models},
author={Sara Babakniya and Ahmed Elkordy and Yahya Ezzeldin and Qingfeng Liu and Kee-Bong Song and MOSTAFA EL-Khamy and Salman Avestimehr},
booktitle={International Workshop on Federated Learning in the Age of Foundation Models in Conjunction with NeurIPS 2023},
year={2023},
}

@inproceedings{chen2022revisiting,
  title={Revisiting Parameter-Efficient Tuning: Are We Really There Yet?},
  author={Chen, Guanzheng and Liu, Fangyu and Meng, Zaiqiao and Liang, Shangsong},
  booktitle={Proceedings of the 2022 Conference on Empirical Methods in Natural Language Processing},
  year={2022}
}

@inproceedings{pu2023empirical,
title={Empirical Analysis of the Strengths and Weaknesses of {PEFT} Techniques for {LLM}s},
author={George Pu and Anirudh Jain and Jihan Yin and Russell Kaplan},
booktitle={International Conference on Learning Representations 2023 Workshop on Mathematical and Empirical Understanding of Foundation Models},
year={2023},
}

@inproceedings{houlsby2019parameter,
  title={Parameter-efficient transfer learning for {NLP}},
  author={Houlsby, Neil and Giurgiu, Andrei and Jastrzebski, Stanislaw and Morrone, Bruna and De Laroussilhe, Quentin and Gesmundo, Andrea and Attariyan, Mona and Gelly, Sylvain},
  booktitle={International Conference on Machine Learning},
  year={2019},
}

@inproceedings{li2021prefix,
  title={Prefix-Tuning: Optimizing Continuous Prompts for Generation},
  author={Li, Xiang Lisa and Liang, Percy},
  booktitle={Proceedings of the 59th Annual Meeting of the Association for Computational Linguistics and the 11th International Joint Conference on Natural Language Processing (Volume 1: Long Papers)},
  pages={4582--4597},
  year={2021}
}

@inproceedings{liu2019signsgd,
  title={sign{SGD} via zeroth-order oracle},
  author={Liu, Sijia and Chen, Pin-Yu and Chen, Xiangyi and Hong, Mingyi},
  booktitle={International Conference on Learning Representations},
  year={2019}
}

@inproceedings{chen2019zo,
  title={{ZO}-{A}da{MM}: Zeroth-order adaptive momentum method for black-box optimization},
  author={Chen, Xiangyi and Liu, Sijia and Xu, Kaidi and Li, Xingguo and Lin, Xue and Hong, Mingyi and Cox, David},
  booktitle={33rd Conference on Neural Information Processing Systems},
  year={2019},
}

@inproceedings{zhang2024revisiting,
  title={Revisiting Zeroth-Order Optimization for Memory-Efficient {LLM} Fine-Tuning: A Benchmark},
  author={Zhang, Yihua and Li, Pingzhi and Hong, Junyuan and Li, Jiaxiang and Zhang, Yimeng and Zheng, Wenqing and Chen, Pin-Yu and Lee, Jason D and Yin, Wotao and Hong, Mingyi and others},
  booktitle={International Conference on Machine Learning},
  year={2024},
}

@inproceedings{hao2024flora,
  title={{FLORA}: Low-Rank Adapters Are Secretly Gradient Compressors},
  author={Hao, Yongchang and Cao, Yanshuai and Mou, Lili},
  booktitle={International Conference on Machine Learning},
  year={2024},
}

@inproceedings{zhao2024galore,
  title={{GaLore}: Memory-Efficient {LLM} Training by Gradient Low-Rank Projection},
  author={Zhao, Jiawei and Zhang, Zhenyu and Chen, Beidi and Wang, Zhangyang and Anandkumar, Anima and Tian, Yuandong},
  booktitle={International Conference on Machine Learning},
  year={2024},
}

@inproceedings{chen2025memory,
  author       = {Yiming Chen and
                  Yuan Zhang and
                  Yin Liu and
                  Kun Yuan and
                  Zaiwen Wen},
  title        = {A Memory Efficient Randomized Subspace Optimization Method for Training
                  Large Language Models},
  booktitle    = {International Conference on Machine Learning},
  year         = {2025},
}

@inproceedings{yeh2023navigating,
  title={Navigating text-to-image customization: From {LyCORIS} fine-tuning to model evaluation},
  author={Yeh, Shih-Ying and Hsieh, Yu-Guan and Gao, Zhidong and Yang, Bernard BW and Oh, Giyeong and Gong, Yanmin},
  booktitle={International Conference on Learning Representations},
  year={2023}
}

@inproceedings{liu2025optimization,
  title={On the Optimization Landscape of Low Rank Adaptation Methods for Large Language Models},
  author={Liu, Xu-Hui and Du, Yali and Wang, Jun and Yu, Yang},
  booktitle={International Conference on Learning Representations},
  year={2025}
}

@inproceedings{he2024subspace,
  author       = {Yutong He and
                  Pengrui Li and
                  Yipeng Hu and
                  Chuyan Chen and
                  Kun Yuan},
  title        = {Subspace Optimization for Large Language Models with Convergence Guarantees},
  booktitle    = {International Conference on Machine Learning},
  year         = {2025},
}

@inproceedings{zhang2024towards,
  title={Towards building the {F}ederated{GPT}: {F}ederated instruction tuning},
  author={Zhang, Jianyi and Vahidian, Saeed and Kuo, Martin and Li, Chunyuan and Zhang, Ruiyi and Yu, Tong and Wang, Guoyin and Chen, Yiran},
  booktitle={IEEE ICASSP},
  pages={6915--6919},
  year={2024},
}

@inproceedings{yan2024frlora,
  author       = {Yunlu Yan and
                  Chun{-}Mei Feng and
                  Wangmeng Zuo and
                  Rick Siow Mong Goh and
                  Yong Liu and
                  Lei Zhu},
  title        = {Federated Residual Low-Rank Adaptation of Large Language Models},
  booktitle    = {International Conference on Learning Representations},
  year         = {2025},
}

@inproceedings{
bai2024federated,
title={Federated Fine-tuning of Large Language Models under Heterogeneous Tasks and Client Resources},
author={Jiamu Bai and Daoyuan Chen and Bingchen Qian and Liuyi Yao and Yaliang Li},
booktitle={38th Conference on Neural Information Processing Systems},
year={2024},
}

@inproceedings{
guo2025selective,
title={Selective Aggregation for Low-Rank Adaptation in Federated Learning},
author={Pengxin Guo and Shuang Zeng and Yanran Wang and Huijie Fan and Feifei Wang and Liangqiong Qu},
booktitle={International Conference on Learning Representations},
year={2025},
}

@inproceedings{xu2023fwdllm,
author = {Xu, Mengwei and Cai, Dongqi and Wu, Yaozong and Li, Xiang and Wang, Shangguang},
title = {Fwd{LLM}: {E}fficient federated finetuning of large language models with perturbed inferences},
year = {2024},
booktitle = {Proceedings of the 2024 USENIX Conference on Usenix Annual Technical Conference},
}

@inproceedings{qin2023federated,
  title={Federated Full-Parameter Tuning of Billion-Sized Language Models with Communication Cost under 18 Kilobytes},
  author={Qin, Zhen and Chen, Daoyuan and Qian, Bingchen and Ding, Bolin and Li, Yaliang and Deng, Shuiguang},
  booktitle={International Conference on Machine Learning},
  year={2024},
}

@inproceedings{zhaoseparate,
  author       = {Hanzhen Zhao and
                  Xingyu Xie and
                  Cong Fang and
                  Zhouchen Lin},
  title        = {{SEPARATE:} {A} Simple Low-rank Projection for Gradient Compression
                  in Modern Large-scale Model Training Process},
  booktitle    = {International Conference on Learning Representations},
  year         = {2025},
}

@inproceedings{shu2024ferret,
  author       = {Yao Shu and
                  Wenyang Hu and
                  See{-}Kiong Ng and
                  Bryan Kian Hsiang Low and
                  Fei Richard Yu},
  title        = {Ferret: Federated Full-Parameter Tuning at Scale for Large Language
                  Models},
  booktitle    = {International Conference on Machine Learning},
  year         = {2025},
}

@inproceedings{kingma2014adam,
  author       = {Diederik P. Kingma and
                  Jimmy Ba},
  title        = {Adam: {A} Method for Stochastic Optimization},
  booktitle    = {International Conference on Learning Representations},
  year         = {2015},
}

@inproceedings{chenenhancing,
  author       = {Yiming Chen and
                  Yuan Zhang and
                  Liyuan Cao and
                  Kun Yuan and
                  Zaiwen Wen},
  title        = {Enhancing Zeroth-order Fine-tuning for Language Models with Low-rank
                  Structures},
  booktitle    = {International Conference on Learning Representations},
  year         = {2025},
}

@inproceedings{koloskova2020unified,
  title={A unified theory of decentralized {SGD} with changing topology and local updates},
  author={Koloskova, Anastasia and Loizou, Nicolas and Boreiri, Sadra and Jaggi, Martin and Stich, Sebastian},
  booktitle={International Conference on Machine Learning},
  year={2020}
}

@inproceedings{wang2018glue,
title={{GLUE}: A Multi-Task Benchmark and Analysis Platform for Natural Language Understanding},
author={Alex Wang and Amanpreet Singh and Julian Michael and Felix Hill and Omer Levy and Samuel R. Bowman},
booktitle={Proceedings of the 2018 {EMNLP} Workshop {B}lackbox{NLP}: Analyzing and Interpreting Neural Networks for {NLP}},
year={2018},
}

@article{liu2019roberta,
  title={{R}o{BERT}a: A robustly optimized bert pretraining approach},
  author={Liu, Yinhan and Ott, Myle and Goyal, Naman and Du, Jingfei and Joshi, Mandar and Chen, Danqi and Levy, Omer and Lewis, Mike and Zettlemoyer, Luke and Stoyanov, Veselin},
  journal={arXiv preprint arXiv:1907.11692},
  year={2019}
}

@article{lialin2023scaling,
  title={Scaling down to scale up: A guide to parameter-efficient fine-tuning},
  author={Lialin, Vladislav and Deshpande, Vijeta and Rumshisky, Anna},
  journal={arXiv preprint arXiv:2303.15647},
  year={2023}
}

@inproceedings{
wang2024flora,
title={{FL}o{RA}: Federated Fine-Tuning Large Language Models with Heterogeneous Low-Rank Adaptations},
author={Ziyao Wang and Zheyu Shen and Yexiao He and Guoheng Sun and Hongyi Wang and Lingjuan Lyu and Ang Li},
booktitle={38th Conference on Neural Information Processing Systems},
year={2024},
}

@inproceedings{
malladi2023fine,
title={Fine-Tuning Language Models with Just Forward Passes},
author={Sadhika Malladi and Tianyu Gao and Eshaan Nichani and Alex Damian and Jason D. Lee and Danqi Chen and Sanjeev Arora},
booktitle={37th Conference on Neural Information Processing Systems},
year={2023},
}

@article{fang2022communication,
  title={Communication-efficient stochastic zeroth-order optimization for federated learning},
  author={Fang, Wenzhi and Yu, Ziyi and Jiang, Yuning and Shi, Yuanming and Jones, Colin N and Zhou, Yong},
  journal={IEEE Transactions on Signal Processing},
  volume={70},
  pages={5058--5073},
  year={2022},
  publisher={IEEE}
}

@inproceedings{hsu2019measuring,
title	= {Measuring the Effects of Non-Identical Data Distribution for Federated Visual Classification},
author	= {Harry Hsu and Hang Qi and Matthew Brown},
year	= {2019},
booktitle = {International Workshop on Federated Learning for User Privacy and Data Confidentiality in Conjunction with NeurIPS 2019}
}

@ARTICLE{gao2025heterogeneity,
  author={Gao, Zhidong and Zhang, Zhenxiao and Zhang, Yu and Gong, Yanmin and Guo, Yuanxiong},
  journal={IEEE Internet of Things Journal}, 
  title={Heterogeneity-Aware Resource Allocation and Topology Design for Hierarchical Federated Edge Learning}, 
  year={2025},
  volume={12},
  number={19},
  pages={39803-39816},
}

@article{zhang2024heterogeneity,
  title={Heterogeneity-aware cooperative federated edge learning with adaptive computation and communication compression},
  author={Zhang, Zhenxiao and Gao, Zhidong and Guo, Yuanxiong and Gong, Yanmin},
  journal={IEEE Transactions on Mobile Computing},
  volume={24},
  number={3},
  pages={2073--2084},
  year={2024},
  publisher={IEEE}
}

@inproceedings{guo2022hybrid,
  title={Hybrid local {SGD} for federated learning with heterogeneous communications},
  author={Guo, Yuanxiong and Sun, Ying and Hu, Rui and Gong, Yanmin},
  booktitle={International Conference on Learning Representations},
  year={2022}
}

@inproceedings{infocommzhenxiao,
  title={Federated Adaptive Fine-Tuning of Large Language Models with Heterogeneous Quantization and {L}o{RA}},
  author={Gao, Zhidong and Zhang, Zhenxiao and Guo, Yuanxiong and Gong, Yanmin},
  booktitle={IEEE International Conference on Computer Communications},
  year={2025},
}
\end{document}